\documentclass[a4paper,11pt]{article}
\usepackage[margin=20mm]{geometry}
\usepackage{hyperref}
\usepackage{url}            
\usepackage{booktabs}       
\usepackage{amsfonts}       
\usepackage{graphicx}
\usepackage{amsmath}
\usepackage{wrapfig}
\usepackage{bbm}
\usepackage{bm}
\usepackage{caption}
\usepackage{amssymb}
\usepackage{mathtools}
\usepackage{amsthm}
\usepackage{subfigure}
\usepackage{algorithmic, algorithm, setspace}
\usepackage{multirow}
\usepackage{appendix}

\DeclareMathOperator*{\argmax}{arg\,max}

\newcommand{\BX}{\mathbf{X}}

\usepackage{soul}
\newcommand{\BZ}{\mathbf{Z}}
\newcommand{\real}{\mathbb{R}}

\newtheorem{theorem}{Theorem}[section]

\newtheorem{remark}{Remark}%
\newcommand{\Hquad}{\hspace{0.5em}} 

\providecommand{\keywords}[1]
{
  \textbf{Keywords---} #1
}

\title{Node Copying: A Random Graph Model for Effective Graph Sampling}

\author{Florence~Regol$^{1\dagger}$, Soumyasundar~Pal$^{1\dagger}$, Jianing~Sun$^{2}$, Yingxue~Zhang$^{2}$, \\ Yanhui~Geng$^{2}$, and~Mark~Coates$^{1}$ \\ \\
$^{1}$Dept. of Electrical and Computer Engineering, McGill University, Montreal, QC, Canada. \\  
$^{2}$Huawei Noah’s Ark Lab, Montreal Research Center, Montreal, QC, Canada.}
\date{}

\begin{document}
\maketitle

\begin{abstract}
\normalsize
There has been an increased interest in applying machine learning
techniques on relational structured-data based on an observed graph. Often, this graph is not fully representative of the true relationship amongst
nodes. In these settings, building a generative model conditioned on the observed
graph allows to take the graph uncertainty into account. Various existing
techniques either rely on restrictive assumptions, fail to preserve topological
properties within the samples or are prohibitively expensive for larger graphs.
In this work, we introduce the node copying model for constructing a distribution
over graphs. Sampling of a random graph is carried out by replacing each
node’s neighbors by those of a randomly sampled similar node. The sampled
graphs preserve key characteristics of the graph structure without explicitly
targeting them. Additionally, sampling from this model is extremely
simple and scales linearly with the nodes. We show the usefulness
of the copying model in three tasks. First, in node classification, a
Bayesian formulation based on node copying achieves higher accuracy in sparse
data settings. Second, we employ our proposed model to mitigate the effect
of adversarial attacks on the graph topology. Last, incorporation of the model in
a recommendation system setting improves recall over {\em state-of-the-art} methods.
\end{abstract}

\noindent \keywords{Generative graph model, graph neural network, adversarial attack, recommender systems }

\section{Introduction}
{\let\thefootnote\relax\footnotetext{$^\dagger$These authors contributed equally to this work.}}
{\let\thefootnote\relax\footnotetext{Corresponding author: Florence Regol (email address: \href{mailto:florence.robert-regol@mail.mcgill.ca}{florence.robert-regol@mail.mcgill.ca})}}
In graph related learning problems, models need to take into account relational structure. This additional information is encoded by a graph that represents data points as nodes and the relationships as edges. In practice, observed quantities are often noisy and the information used to construct a graph is no exception. The provided graph is likely to be incomplete and/or contain spurious edges. For some graph learning tasks, this error or incompleteness is explicitly stated. For example, in the recommendation system setting, the task is to infer unobserved links representing users' preferences for items. In other cases, such as protein-protein interaction networks, graph uncertainty is implicit, arising because the graphs are constructed from noisy measurement data.

In these cases, we can view the graph as a random quantity. The modelling of random graphs as realisations of statistical models has been widely studied in network analysis~\cite{Goldenberg2010}. Parametric random graph models have been used extensively to study complex systems and perform
inference~\cite{albert2002, erdos59, airoldi2009}. Parameters are usually estimated based on the observed graph. However, these models often fail to capture the structural properties of real-world networks and as a result, are not suitable as generative models. In addition, inference of parameters can be prohibitively expensive for large scale networks, and many parametric models cannot take into account node or edge attributes.

Recent interest in generative models for graph learning has resulted in auto-encoder based formulations~\cite{kipf2016,grover2016}. There has also been an effort to combine the strengths of parametric random graph models and machine learning approaches~\cite{metha2019}. In these models, the probability of the presence of an edge is related to the node embeddings of the incident nodes. Although the auto-encoder solutions excel in their targeted objective of link prediction, they fail to generate representative samples~\cite{bojchevski2018}.
In~\cite{bojchevski2018, wang2017}, application of generative
adversarial networks (GANs) is considered
for graph generative models. \cite{bojchevski2018} shows
that the GAN generated graphs do not deviate from the observed graph in terms of structural properties. These approaches are promising but the
required training time can be prohibitive for graphs of even a moderate size.

The main contribution of this paper is to introduce a novel generative
model for graphs for the case when we have access to a single observed graph. The generative model is designed to allow the sampling of multiple graphs that can be used as an ensemble to improve performance in graph-based learning tasks. We call the approach ``node copying". Based on a similarity metric between nodes, the node copying model generates 
random graph realizations by replacing the edges of a node in the observed graph with those
of a randomly chosen similar node. 
The assumption made by this model is that {\em similar} nodes should have {\em similar} neighborhoods. The meaning of `similar' varies depending on how the model is employed in a learning task and has to be defined according to the application setting. Most graph learning algorithms rely on and exploit the property of homophily, which implies that nodes with similar characteristics are more likely to be connected to one another~\cite{zhu2020}. The proposed node copying model preserves this homophily (one similar neighbourhood is swapped for another similar neighbourhood). The advantage is that the generative model permits sampling of an ensemble of graphs that can be used to improve the performance of learning algorithms, as illustrated in later sections of the paper.

By construction, each sampled graph possesses the same set of nodes as the observed graph and the identities of the nodes are retained throughout the sampling procedure. The identity of a node is specified by a label and/or a set of node features; hence, each node preserves its own features or labels (if available) in the sampled graphs. The proposed model is simple
and the computational overhead is minimal. Despite its simplicity, we show that incorporating such a model can be beneficial in a range of graph-based learning tasks. The model is flexible in the sense that the choice of the similarity
metric is adapted to the end-to-end learning
goal. It can be combined with {\em state-of-the-art} methods to provide an
improvement in performance. In summary, the proposed
model: 1) \textit{generates useful samples} which improve the
performance in graph-based machine learning tasks and retain
important structural properties; 2) \textit{is flexible} so that it is
applicable to various tasks (e.g., the bipartite graphs in
recommendation systems are drastically different from the community
structured graphs considered in node classification); and 3)
\textit{is fast} both for model construction and random sample generation.

Extensive experiments demonstrate the effectiveness of our model for
three different graph learning tasks. We show that the incorporation
of the model in a Bayesian framework for node classification leads to
improved performance over {\em state-of-the art} methods for the scarce data
scenario. The model has much lower computational complexity compared
to other Bayesian alternatives. In the face of adversarial attacks, a
node copying based strategy significantly mitigates the effect of the
corruption. Lastly, the use of the node copying model in a personalized
recommender system leads to improvement over a {\em state-of-the-art} graph-based method. Preliminary results for application of the node copying model in semi-supervised node classification and in defense against adversarial attacks were published in~\cite{pal2019} and~\cite{regol2019} respectively. In this paper, we add theoretical results to shed light on the statistical properties of the sampled graphs in special cases, conduct thorough experimentation to examine the similarity of the sampled graphs to the observed graph, show the effectiveness of the proposed approach in a data-scarce classification problem and against diverse adversarial attacks, and extend the application of the copying model to the recommendation setting.

\section{Related work}

\textbf{Parametric random graph models}: There is a rich history of
research and development of parametric models of random graphs. These models have been designed to generate graphs exhibiting specific
characteristics and their theoretical properties have been studied in depth. They can yield samples representative of an observed
graph provided that the model is capable of representing the
particular graph structure and parameter inference is successful. The
Barabasi-Albert model~\cite{albert2002}, exponential random
graphs~\cite{erdos59}, exchangeable random graph
models~\cite{caron2017}, and the large family of stochastic block
models (SBMs)~\cite{nowicki2001}, including mixed-membership
SBMs~\cite{airoldi2009}, degree-corrected SBMs~\cite{peng2016}, and
overlapping SBMs~\cite{latouche2009, kurt2009}), fall into this category. Configuration model~\cite{fosdick2018, casiraghi2018} variants preserve the degree sequence of a graph while changing the connectivity pattern in the random samples. \cite{drobyshevskiy2019} provides a recent survey of various random graph models. A major drawback of these models is that they impose relatively strict
structural assumptions in order to maintain tractability. As a result, they often cannot model characteristics like large clustering coefficients~\cite{bringmann2019}, small world connectivity and exponential degree distributions~\cite{veitch2015}, observed in many real-world networks. Additionally, most cannot readily take into account node or edge attribute information, and high-dimensional parameter inference can  be prohibitively computationally expensive for larger graphs. 

\noindent\textbf{Learning-based models}: Other generative graph models have
emerged from the machine learning community, incorporating an
auto-encoder structure. The models are commonly trained to accurately
predict the links in the graph, and as a result, tend to fail to
reproduce global structural properties of the observed
graph. \cite{kipf2016} introduces a variational auto-encoder based
model parameterized by a graph convolutional network (GCN) to learn
node embeddings. The link probabilities are derived from the dot product of the obtained node embeddings. \cite{pan2018adversarially} adopts a similar approach, but employs adversarial regularization to learn more robust embeddings. Both of these models exhibit impressive clustering of node embeddings. \cite{grover2019} adds a message passing component to the decoder that is based on intermediately learned graphs. This leads to improved representations and better link prediction.  \cite{metha2019} combines the strengths of the parametric models and the graph-based learning methods, proposing the DGLFRM model, which aims to retain the interpretability of the overlapping SBM (OSBM) paired with the flexibility of the graph auto-encoder. The incorporation of the OSBM improves the ability of the model to capture block-based community structure.

An alternative approach is to use generative adversarial networks (GANs) as the basis for graph models. \cite{wang2017} models edge probability through an adversarial framework in the GraphGAN model. The NetGAN model represents the graph as a distribution on random walks~\cite{bojchevski2018}. Compared to auto-encoder based methods, the GAN based methods seem more capable of capturing the structural characteristics of an observed graph. The major disadvantage is that the models are extremely computationally demanding to train and the success of the training can be sensitive to the choice of hyperparameters.

Our focus in this paper is on learning a graph model based on a single observed graph. By contrast, there is a growing body of work that focuses on learning graph models that can reproduce graphs that have characteristics similar to a dataset of multiple training graphs. These approaches can preserve important structural attributes of the graph(s) in the dataset, but the sampled graphs do not retain node identity information. They cannot be applied in the node- and edge-oriented learning tasks we focus on. In this category, there have been variational auto-encoder approaches~\cite{simonovsky2018, ma2018rgvae}, GAN-based approaches~\cite{de2018}, models based on iterative generation~\cite{li2018b}, and auto-regressive models~\cite{you2018,liao2019}.
\section{Node copying model}
We propose to build a random graph model by introducing random perturbations to an initial observed graph $\mathcal{G}_{obs}$. Our aim is to generate graphs $\mathcal{G}$ that are close to $\mathcal{G}_{obs}$ in some structural sense, while preserving any metadata associated with a node's identity (e.g., node features, labels). Here, we consider that $\mathcal{G}_{obs} = \{\mathcal{V}_{obs}, \mathcal{E}_{obs}\}$ is a directed graph (the extension to the undirected case is straightforward and will be explained at the end of this section). $\mathcal{V}_{obs}$ denotes the set of $N$ nodes and $\mathcal{E}_{obs}$ is the set of directed edges of the form $(i, j, A_{\mathcal{G}_{obs}, ij})$ which indicates that there is a directed edge from node $i$ to node $j$ with edge weight encoded in the adjacency matrix $A_{\mathcal{G}_{obs}}  \in \real_+^{N \times N} $.  For the sampled graph $\mathcal{G}= \{\mathcal{V}, \mathcal{E}\}$, we have the same set of vertices as  $\mathcal{G}_{obs}$ i.e. $\mathcal{V} = \mathcal{V}_{obs}$ but the connectivity pattern is different i.e. $\mathcal{E} \neq \mathcal{E}_{obs}$.

To that end, we introduce a discrete perturbation random vector  $\boldsymbol{\zeta} \in S_{\boldsymbol{\zeta}}^N$, whose entries can have values from a finite or countably infinite set $S_{\boldsymbol{\zeta}}$ and define a mapping $T : \real_+^{N \times N} \times S_{\boldsymbol{\zeta}}^N \to \real_+^{N \times N}  $ whose output is an adjacency matrix $A_{\mathcal{G}} \in \real_+^{N \times N}$ of the same dimension as the observed adjacency matrix $A_{\mathcal{G}_{obs}}\in \real_+^{N \times N} $, based on the inputs $A_{\mathcal{G}_{obs}} $ and  $\boldsymbol{\zeta}$. We require $\mathcal{G}$ to have the same set of nodes as $\mathcal{G}_{obs}$; hence, $\mathcal{G}$ is fully characterized by $A_{\mathcal{G}}$.

The mapping is not necessarily one-to-one, i.e., multiple $\boldsymbol{\zeta}$s can generate the same graph $\mathcal{G}$. The probability of generating a specific graph $\mathcal{G}$ is then specified by defining a probability distribution for $\boldsymbol{\zeta}$. We model the distribution of  $\boldsymbol{\zeta}$ to be conditioned on $\mathcal{G}_{obs}$ and possible additional information $\mathcal{D}$ (e.g. node/edge labels or features). We therefore obtain the following conditional probability:
\begin{align}
p(\mathcal{G}|\mathcal{G}_{obs},\mathcal{D}) = \sum_{\boldsymbol{\zeta}} p(\boldsymbol{\zeta}|\mathcal{G}_{obs},\mathcal{D}) \Hquad \mathbbm{1}[T(A_{\mathcal{G}_{obs}},\boldsymbol{\zeta})=A_{\mathcal{G}}]\,, \label{eqn:model}
\end{align}
where $\mathbbm{1}[\cdot]$ denotes the indicator function, which takes the value one if $T(A_{\mathcal{G}_{obs}},\boldsymbol{\zeta})=A_{\mathcal{G}}$ and is zero otherwise. This is the foundation of our proposed copying model. 
We propose a `node copying' mechanism to implement the mapping $T$. The intuition for the node copying mechanism is as follows.  Suppose two nodes, $i$ and $j$, are `similar', where `similarity' depends on the context or application setting. It may be defined in terms of node attributes, node labels, structural properties, or a combination of all of these. Then we conjecture that similar nodes also have similar neighbourhoods. For example, suppose both nodes have the same label. Then, in a homophilic graph, most of their neighbours also have the same label, and the neighbourhoods are thus similar. If we replace node $i$'s neighbourhood with that of node $j$ (the copy operation), then we obtain a new graph, but the homophilic nature of the graph is preserved.

To make our random graph model concrete, we need to specify 1) the mapping $T$ with the nature of $\boldsymbol{\zeta}$ and 2) the conditional distribution $p(\boldsymbol{\zeta}|\mathcal{G}_{obs},\mathcal{D})$.

\begin{wrapfigure}{l}{0.415\textwidth}
\centering
\includegraphics[scale = 0.1]{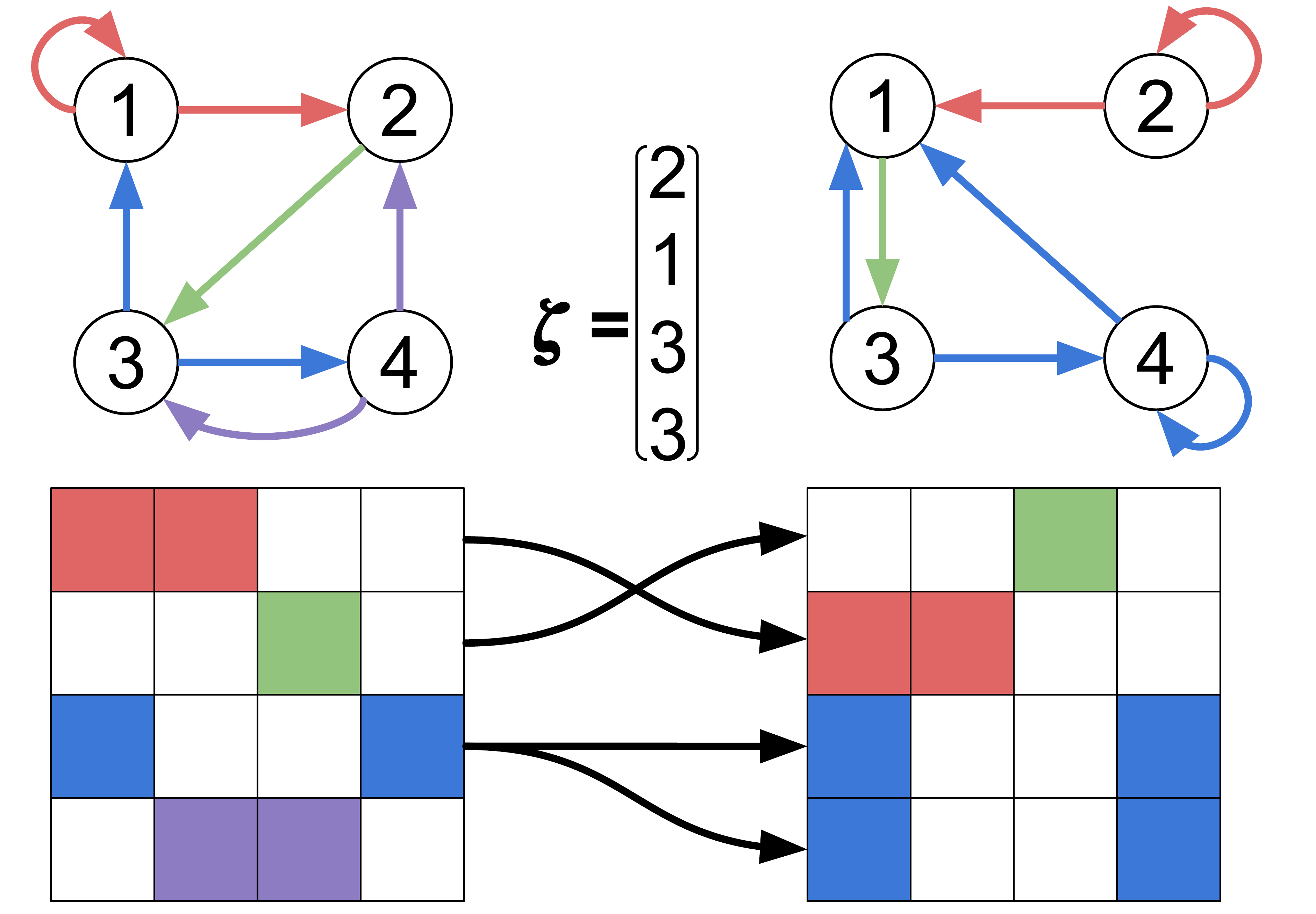}
\caption{Application of the node copying $T$ on observed graph $\mathcal{G}_{obs}$ (left-side) for a given $\boldsymbol{\zeta}$ to obtain $\mathcal{G}$ (right-side).}
\label{fig:copying}
\end{wrapfigure}
\emph{1) Mapping $T(\cdot,\cdot)$ and random vector $\boldsymbol{\zeta}$:} We construct a \noindent \textbf{node copying mapping} $T$ that replaces the neighborhood of a node $i$  by that of another node given by the $i$-th entry of the vector $\boldsymbol{\zeta} = [\zeta^1, \zeta^2, ... \zeta^N]^\top \in \{1,2,... N\}^N$. $\zeta^i$ is referred to as the \textit{replacement node} of node $i$.  An explicit expression for $T$ can be provided using a selection matrix $C_{\boldsymbol{\zeta}} \in \{0,1\}^{N \times N}$, where $C_{\boldsymbol{\zeta}}[i,j] = 1$ if $ j = \zeta^i$ and 0 otherwise. Then $T(A_{\mathcal{G}_{obs}},\boldsymbol{\zeta}) = C_{\boldsymbol{\zeta}}A_{\mathcal{G}_{obs}} =  A_{\mathcal{G}} $. Figure~\ref{fig:copying} depicts the application of the mapping $T$ on example $\mathcal{G}_{obs}$ and $\boldsymbol{\zeta}$ to obtain a graph $\mathcal{G}$.  

\emph{2) Distribution $p(\boldsymbol{\zeta}|\mathcal{G}_{obs},\mathcal{D})$:}
The distribution for $\boldsymbol{\zeta}$ should encode the notion of node similarity, and is therefore task dependant. The entries in $\boldsymbol{\zeta}$ are assumed to be mutually independent
to facilitate inference: $p(\boldsymbol{\zeta}|\mathcal{G}_{obs},\mathcal{D}) = \prod_{j=1}^{N}p(\zeta^j|\mathcal{G}_{obs}, \mathcal{D})$.
We provide examples of candidate probability distributions in later sections. One approach is to train a node embedding algorithm to learn $\{\mathbf{e}_i\}_{i=1:N} = f_{\mathbf{\Theta}}(\mathcal{G}_{obs},\mathcal{D})$, where $\mathbf{e}_i$ denotes the representation of node $i$. The probability $p(\zeta^j = i|\mathcal{G}_{obs}, \mathcal{D}) \propto \text{sim} (\mathbf{e}_j,\mathbf{e}_i)$ to promote frequent copying between nodes if they are `similar', for some function $\text{sim}(\cdot, \cdot)$ , which measures similarity between two nodes.


This completes the description of our proposed model. Sampling a graph $\mathcal{G}$ from the node copying model involves sampling each element of $\boldsymbol{\zeta}$ independently according to $p(\zeta^j|\mathcal{G}_{obs}, \mathcal{D})$. Once $\boldsymbol{\zeta}$ is constructed, we replace the out-edges of the $i$-th node in the observed graph $\mathcal{G}_{obs}$ by those of the $\zeta^i$-th node. In general, we construct the distribution $p(\zeta^j|\mathcal{G}_{obs}, \mathcal{D})$ so that it is simple to sample from.  We thus avoid any need to use MCMC or importance sampling for drawing the sampled graphs. For using this model to sample undirected graphs when $\mathcal{G}_{obs}$ is undirected, we apply the following procedure. We treat $\mathcal{G}_{obs}$ as a directed graph and use the node copying model to sample a directed graph $\mathcal{G}'$. Then, we change each directed edge to an undirected edge to obtain an directed graph $\mathcal{G}$ by setting $A_{\mathcal{G}, ij} = \text{max} (A_{\mathcal{G}', ij}, A_{\mathcal{G}', ji})$.

\begin{remark}
Since the model is defined in terms of perturbations of the observed graph $\mathcal{G}_{obs}$, the utility of the proposed graph model hinges on $\mathcal{G}_{obs}$ being a meaningful starting point for inference. In some cases, even if no graph is directly available for a problem, we can learn relationships between entities and construct $\mathcal{G}_{obs}$ from the available data. Computationally efficient techniques for such graph construction are reported in~\cite{dong2016, kalofolias2017}. In some cases, the computational requirements can still be considerable.  
\end{remark}

\section{Analysis and comparison of sampled graphs}\label{sec:analysis}
In order to gain insight into the properties of the random graphs
drawn from the node copying model, we evaluate several properties of the
samples and compare with other generative methods and
the observed graph, for four datasets.\\
\noindent \textbf{Node Embedding Baselines}: We compare with various graph-based auto-encoder models, which are
trained to optimize an unsupervised variational objective for a link prediction task. Specifically we compare with the
Variational Graph Auto-Encoder (VGAE)~\cite{kipf2016}, GRAPHITE~\cite{grover2019} and
DGLFRM-B~\cite{metha2019} algorithms. In each algorithm, the
probability of an edge between any two nodes depends on the learned embeddings. The probability of an edge between nodes $i$ and $j$ from a specific model is denoted as $p_{ij}^{model}$. Instead of their originally proposed use for link prediction, we wish to evaluate these models as generative models. Hence, we train the models using the entire graph and estimate $\{p_{ij}^{model}\}_{1 \leq i \leq j \leq N}$. We refer to these probabilities as `raw' probabilities and sample graphs according to them. In practice, this results in dense sampled graphs where the edge densities in the sampled graphs are multiple orders-of-magnitude higher than the density of the observed graph $\mathcal{G}_{obs}$. As a result, many of the statistics of the sampled graphs are not comparable with those of $\mathcal{G}_{obs}$. In order to sample more realistic graphs, we employ logistic regression based calibration. We fit a model
$p_{ij}^{cal} = p(A_{obs,ij}=1|p_{ij}^{model}) = \sigma(\alpha p_{ij}^{model} +
\beta)$ (where $\sigma(\cdot)$ denotes the sigmoid function) to learn $(\alpha, \beta)$ and use the predicted probabilities for sampling graphs. Even after calibration the densities of the sampled graphs are still too high, so we explicitly re-scale the calibrated probabilities such that the sampled graphs in expectation have the same number of edges $|\mathcal{E}_{obs}|$ as $\mathcal{G}_{obs}$: $p^{cc}_{ij} = \displaystyle{\frac{|\mathcal{E}_{obs}|p_{ij}^{cal}}{\sum_{i=1}^N\sum_{j=1}^N p_{ij}^{cal}}}$. We call this procedure `calibration and correction' (cc).

\noindent \textbf{Node copying}: Since the baselines are unsupervised
techniques, for a fair comparison we refrain from using node
labels for the node copying model. Instead, we sample replacement
nodes according to the distances between learned embeddings. The selected datasets are used as benchmark for node classification task and are expected to  exhibit some homophilic property, which suggests that nodes with the same label are clustered together and share similar features. This indicates that it is sensible to construct a similarity metric using node embeddings that are derived from both topological and feature information.
Specifically, for any node, we sample the replacement node uniformly
from the $K$-nearest neighbors according to the Euclidean distance between
embeddings. 

\begin{table}[ht]
\centering
\caption{Average statistics of 100 sampled graphs. Bolded entries indicate the closest values to the observed value. N/A indicates entries that were not reported due to processing limitations.}
\setlength{\tabcolsep}{3pt}
\resizebox{\columnwidth}{!}{
\begin{tabular}{lccccc|ccccc} \midrule[0.25ex]
&\textbf{\begin{tabular}[c]{@{}c@{}}Avg. \\ degree  \end{tabular}} &\textbf{\begin{tabular}[c]{@{}c@{}}Max.\\ degree\end{tabular}} &\textbf{\begin{tabular}[c]{@{}c@{}}Cross \\ com. $\%$\end{tabular}} &\textbf{\begin{tabular}[c]{@{}c@{}}Claws $\%$ \\ $ (\times 10^{-4})$ \end{tabular} }  &\textbf{\begin{tabular}[c]{@{}c@{}}Edge  dist.\\ ent. $ (\%)$\end{tabular}}&\textbf{\begin{tabular}[c]{@{}c@{}}Avg. \\ degree  \end{tabular}} &\textbf{\begin{tabular}[c]{@{}c@{}}Max.\\ degree\end{tabular}} &\textbf{\begin{tabular}[c]{@{}c@{}}Cross \\ com. $\%$\end{tabular}} &\textbf{\begin{tabular}[c]{@{}c@{}}Claws $\%$\\ $ (\times 10^{-4})$ \end{tabular} }  &\textbf{\begin{tabular}[c]{@{}c@{}}Edge  dist.\\ ent. $ (\%)$\end{tabular}}  \\ \midrule
&\multicolumn{5}{c}{\textbf{Cora}} &\multicolumn{5}{c}{\textbf{Polblogs}} \\ 
\midrule
\textbf{Observed} &3.89 &168 &19.6 &6.34 &95.6 &27.4 &351 &9.42 &10.1 &90.3\\ \hline
\textbf{GRAPHITE} (cc)&\textbf{3.9} &24 &34.4 &0.43 &\textbf{96} &\textbf{27.3} &64 &18.1 &0.91 &98.9\\
\textbf{raw} &$1.3\times10^3$ &$1.4\times10^3$ &77.6 &0.18 &100 &614 &712 &39.2 &0.67 &100 \\
\textbf{calibrated} &105 &519 &34.5 &0.44 &97.7 &256 &530 &18.1 &0.91 &99.2 \\ \hline
\textbf{GVAE} (cc)&\textbf{3.9} &22.1 &37 &0.302 &96.2 &\textbf{27.3} &64.1 &18.2 &0.918 &98.9 \\
\textbf{raw} &1.34$\times 10^3$ &1.44$\times 10^3$  &77.7 &0.136 &100 &614 &714 &39.4 &0.673 &100 \\
\textbf{calibrated} &125 &587 &37 &0.309 &97.9 &255 &533 &18.2 &0.918 &99.1 \\ \hline
\textbf{DGLFRM-B}  (cc)  &\textbf{3.9} &14.9 &38.8 &0.169 &97.8 &\textbf{27.3} &58.8 &18.1 &0.874 &99 \\
\hline \hline
\textbf{COPYING} K=5 &3.5 &\textbf{105 } &\textbf{ 18.4} &\textbf{3.2} &96.9 &26 &\textbf{311} &\textbf{8.94} &\textbf{8.61} &\textbf{90.6} \\
\textbf{COPYING} K=10 &3.3 &63.9  &20.4 &1.22 &97.4 &25.5 &290 &8.83 &7.72 &90.8 \\
\textbf{COPYING} K=15 &3.2 &58.1 &22 &1.1 &97.5 &25.2 &274 &8.72 &7.33 &90.9 \\
\hline   
&\multicolumn{5}{c}{\textbf{Bitcoins}} &\multicolumn{5}{c}{\textbf{Amazon-photo}} \\ \hline
\textbf{Observed} &2.08 &53 &27.1 &240 &92.3 &31.8 &$1.43\times10^3$ &17.3 &0.75 &94.2 \\\hline
\textbf{GRAPHITE} (cc) &\textbf{2.09} &9.21 &23.8 &6.22 &\textbf{94.7} &\textbf{31.8} &87.2 &81.3 &0.02 &99.3 \\
\textbf{raw} &250 &311 &37.6 &4.97 &99.8 &$4.4\times 10^3$ &$6.0\times 10^3$ &83.1 &0.02 &100 \\ 
\textbf{calibrated} &84.8 &219 &23.8 &9.48 &99.1 &$2.7\times10^3$ &$6.4\times10^3$ &81.3 &0.02 &99.4\\ \hline

\textbf{GVAE} (cc)&\textbf{2.09} &9.18 &22.2 &6.24 &94.8 &\textbf{31.8} &71.7 &80.2 &0.0212 &99.5 \\
\textbf{raw} &236 &285 &37.1 &4.62 &99.9 &5.41$\times 10^3$ &6.24$\times 10^3$ &82.4 &0.0226 &100 \\
\textbf{calibrated} &72.4 &198 &22.2 &6.96 &99.1 &2.95$\times 10^3$ &4.84$\times 10^3$ &80.2 &0.0228  &99.7 \\ \hline
\textbf{DGLFRM-B} (cc)&235 &291 &36.9 &4.59 &99.9 &N/A &N/A &N/A &N/A &N/A\\
\hline \hline
\textbf{COPYING }K=5 &1.8 &\textbf{34.9} &\textbf{23.6} &\textbf{125} &95.3 &31.5 &$1.04\times10^3$ &\textbf{41.5} &\textbf{0.532} &\textbf{94.7} \\

\textbf{COPYING} K=10 &1.73 &28.3 &23 &89.9 &96.0 &31.1 &$\mathbf{1.05 \times 10^3}$ &47.2 &0.502 &94.9 \\
\textbf{COPYING }K=15 &1.68 &21  &22.6 &50.3 &96.7 &31 &$1.03\times10^3$ &50.6 &0.502 &94.9 \\
\bottomrule
\end{tabular}}
\label{tab:graph_stat}
\end{table}
\noindent \textbf{Experiment}: The qualitative similarity of two graphs can be measured in multiple ways. To obtain an insight into which method can generate graphs that are the closest to an original graph, we report various graph statistics on the sampled graphs and compare them to the statistics of the observed graphs for Cora, Polblogs, Bitcoin, and Amazon-photo datasets (homophilic graphs, used extensively for node classification) in Table~\ref{tab:graph_stat}. Details of all datasets used in this paper are included in~\ref{ap:data}. 
For uniformity, the directed graph of Bitcoin has been made undirected by casting every in/out edge as an undirected edge.

\noindent \textbf{Metrics}: For a graph $\mathcal{G}$,  $d(v)$ is the degree of node $v$ and $c_v$ denotes the class label of node $v$. We report the average and minimum degree, the proportion of cross-community edges: $\frac{1}{|\mathcal{E}|}\sum_{(i,j) \in \mathcal{E}} \mathbbm{1}[c_i \neq c_j] $ , the proportion of claws (`Y' structure): $\frac{1}{\binom{|\mathcal{E}|}{3}}\sum_{v \in \mathcal{V} }{ \binom{d(v)}{3} }$  and the relative entropy of the
edge distribution: $\frac{1}{\log N} \sum_{v \in \mathcal{V}} \frac{-d(v)}{|\mathcal{E}|}\log \big(\frac{d(v)}{|\mathcal{E}|} \big)$. By reporting various summary graph statistics, we can have a broad picture of the properties of the generated graphs.

In each case, 10
random trials of the model training are conducted and 10 random graphs
are sampled based on each trial. The results are thus averaged over 10 $\times$
10 = 100 sampled graphs. We observe that without the `calibration and correction', the baseline node embedding algorithms fail to generate representative graphs as the characteristics of the sampled graphs show large deviation from those of the observed graph. In particular, in `raw' versions, we obtain graph samples which are significantly denser, have higher proportion of cross-community links, have lower proportion of claws, and possess a flatter degree distribution compared to the observed graph. This shows that learning node embeddings while targeting a link prediction task does not necessarily result in a good generative model.

`Calibration' reduces this effect and leads to sparser graphs with
better cross-community structure. After the explicit `correction', the average degree of the sampled graphs is the same as that of the observed graph. However, the degree distribution still does not match with that of the observed graph, as evident from the maximum degree, as well as the relative entropy of the edge distribution. Moreover, the sample graphs have a much lower proportion of claws. This is not unexpected, because any model which is based on independence of the links cannot model such local structural dependencies observed in real world graphs. On the other hand, the samples from the proposed node copying model show a much better agreement with the observed graph in terms of all the statistics we consider. In particular, a much lower proportion of cross-community links in the graph samples of the copying model provides a rationale for its efficacy when the model is adopted for a node classification task. Overall, these tendencies can be observed when sampling is performed using a different number of nearest neighbors for the sampling distribution ($K=5,10,15$).  In this experiment, sampling from the $K=5$ nearest neighbors seems to be the best choice.

\noindent \textbf{Connections to other random graph models:} We further support this notion that the copying model preserves the structure of the observed graph with theoretical results. We show that if the observed graph itself is a sample from some parametric random graph model, then suitable copying procedures can ensure that the sampled graphs have the same marginal distribution. We prove this result for two popular random graph models, namely for the Stochastic Block Model in Theorem~\ref{th:sbm} and for the  Erd\H{o}s-R\'{e}nyi family in Theorem~\ref{th:er}.

\begin{theorem}
\label{th:sbm}
Let $\mathcal{G}_{obs}$ be a sample of a Stochastic Block Model (SBM) with $N$ nodes and $K$ communities and symmetric conditional probability matrix $\boldsymbol{\beta} \in \real^{K \times K}$. We use a node copying procedure between nodes of the same label, i.e., a node $j$ cannot be copied in place of node $i$, if the labels $c_i \neq c_j$. Then any sampled graph $\mathcal{G}_{sa}$ is marginally another realization from the same SBM. 
\end{theorem}

\begin{proof}
For $\mathcal{G}_{obs}$, we have $p(A_{ij}^{obs}=1|c_i, c_j ) = \beta_{c_i,c_j}$. Now, based on the copying strategy above, the probability of  observing an edge $(i,j)$ in $\mathcal{G}_{sa}$ is given by:
\begin{align}
p(A_{ij}^{sa}=1| c_i ,c_j) &=  \sum_{v=1}^N p(A_{vj}^{obs}=1|c_i, c_j, \zeta^i = v) p(\zeta^i = v|c_i, c_v) \,,\nonumber\\
&=  \sum_{v : c_i = c_v} p(A_{vj}^{obs}=1|c_i, c_j,\zeta^i = v)p(\zeta^i = v|c_i,c_v) \, \quad \text{ (since } p(\zeta^i = v|c_i \neq c_v) = 0),  \nonumber\\
&=  \beta_{c_i,c_j}  = p(A_{ij}^{obs}=1|c_i, c_j ), 
\end{align}
since, $p(A_{vj}^{obs}=1|c_i, c_j,\zeta^i = v) = \beta_{c_i,c_j} = p(A_{ij}^{obs}=1|c_i, c_j )$ for all $\{{v : c_i = c_v}\}$ and $\sum_{v : c_i = c_v} p(\zeta^i = v|c_i,c_v) = 1$. This is reliant on only copying within the same community, as $p(\zeta^i = v|c_i \neq c_v) = 0$.
Next we consider the joint distribution of two edges, i.e. $p(A_{i_1j_1}^{sa}=1, A_{i_2j_2}^{sa}=1 | c_{i_1} , c_{i_2}, c_{j_1}, c_{j_2})$. We need to consider the following two cases:\\
\noindent \textbf{Case 1: $i_1 \neq i_2$}
\begin{align}
&p(A_{i_1j_1}^{sa}=1, A_{i_2j_2}^{sa}=1 | c_{i_1} , c_{i_2}, c_{j_1}, c_{j_2}) \nonumber \\ &=\sum_{v=1}^N \sum_{u=1}^N p(\zeta^{i_1} = v, \zeta^{i_2} = u|c_{i_1},c_{i_2})
p(A_{vj_1}^{obs}=1, A_{uj_2}^{obs}=1|c_{i_1} c_{i_2}, c_{j_1}, c_{j_2}, \zeta^{i_1} = v, \zeta^{i_2} = u)\,,\nonumber\\
&= \sum_{\substack{v : \\c_{i_1} = c_{v}}} \sum_{\substack{ u : \\c_{i_2}= c_{u}}} p(\zeta^{i_1} = v,|c_{i_1}) p(\zeta^{i_2} = u,|c_{i_2} )  p(A_{vj_1}^{obs}=1|c_{i_1}, c_{j_1}, \zeta^{i_1} = v) p(A_{uj_2}^{obs}=1|c_{i_2}, c_{j_2}, \zeta^{i_2} = u)\,,\nonumber\\
&= \beta_{c_{i1}c_{j1}}\beta_{c_{i2}c_{j2}} = p(A_{i_1j_1}^{obs}=1, A_{i_2j_2}^{obs}=1 | c_{i_1}, c_{i_2}, c_{j_1}, c_{j_2})\,.
\end{align}
\noindent \textbf{Case 2: $i_1 = i_2$} 
\begin{align}
&p(A_{i_1j_1}^{sa}=1, A_{i_2j_2}^{sa}=1 | c_{i_1} = c_{i_2}, c_{j_1}, c_{j_2}) \, \nonumber \\ &=\sum_{v=1}^N p(\zeta^{i_1} = v|c_{i_1}) p(A_{vj_1}^{obs}=1, A_{vj_2}^{obs}=1|  c_{i_1}=c_{i_2}, c_{j_1}, c_{j_2}, \zeta^{i_1} = v)\,,\nonumber\\
&=\sum_{v : c_{i_1} = c_{v}}p(\zeta^{i_1} = v,|c_{i_1})  p(A_{vj_1}^{obs}=1|c_{i_1} , c_{j_1}, \zeta^{i_1} = v) p(A_{vj_2}^{obs}=1|c_{i_1}, c_{j_2}, \zeta^{i_1} = v)\,,\nonumber\\
&=\beta_{c_{i1}c_{j1}}\beta_{c_{i1}c_{j2}} = p(A_{i_1j_1}^{obs}=1, A_{i_1j_2}^{obs}=1 | c_{i_1}=c_{i_2}, c_{j_1} , c_{j_2})\,.
\end{align}
Here, the joint distribution of $\boldsymbol{\zeta}$ factorizes because of the independence of $\zeta^i$s. The joint conditional edge probability in $\mathcal{G}_{obs}$ is the product of individual conditional edge probabilities as $\mathcal{G}_{obs}$ is sampled from a SBM. Similarly, if we consider any arbitrary subset of the edges in $\mathcal{G}_{sa}$, we can show  that the joint distribution is mutually independent over the edges and is the same as that of $\mathcal{G}_{obs}$. This proves the theorem.
\end{proof}

\begin{theorem}
\label{th:er}
If $\mathcal{G}_{obs}$ is a sample of an Erd\H{o}s-R\'{e}nyi model with $N$ nodes and connection probability $\theta \in [0, 1])$, then for any arbitrary distribution of $\zeta$, $\mathcal{G}_{sa}$ is also a sample from the same model.
\end{theorem}

\begin{proof}
For $\mathcal{G}_{obs}$, we have $p(A_{ij}^{obs}=1|\theta) = \theta$. Now, based on any arbitrary copying strategy, we have
\begin{align}
p(A_{ij}^{sa}=1| \theta) &= \sum_{k=1}^N p(\zeta^i = k|\theta) p(A_{kj}^{obs}=1|\theta, \zeta^i = k)\,,\nonumber\\
&= \theta = p(A_{ij}^{obs}=1|\theta)\,,
 \end{align}
since $p(A_{kj}^{obs}=1|\theta, \zeta^i = k) = p(A_{kj}^{obs}=1|\theta) = \theta$ for $1 \leq k \leq N$ and $\sum_{k=1}^{N} p(\zeta^i = k|\theta) = 1$. Similarly, if we consider any arbitrary subset of the edges in $\mathcal{G}_{sa}$, we can show that the joint distribution is mutually independent over the edges and is the same as that of $\mathcal{G}_{obs}$. This proves the theorem.
\end{proof}

Although generating samples that preserve graph structure can be useful, it is not the main purpose of our model. In the following three sections, we explain how  this model can be applied in node classification, in protection against adversarial attack and in a recommendation system.

\section{Application -- Node copying for semi-supervised node classification}
The node classification algorithms that we consider rely heavily on the homophily of the graph structure in predicting node labels. Nodes of the same label tend to be connected together more often than the nodes of different labels. If we define a notion of similarity which is based on the node labels, a node copying model can sample graphs which preserve the homophily property (see Theorem~\ref{th:sbm}). Those graph samples can then be integrated in a Bayesian framework to make better prediction.

\noindent \textbf{Problem Setting:} In the task of semi-supervised node classification, apart from the observed graph $\mathcal{G}_{obs}$, we have access to node features $\BX$ and the labels in the training set $\mathbf{Y_{\mathcal{L}}}$. So, $\mathcal{D} = (\BX, \mathbf{Y_{\mathcal{L}}})$. The goal is to infer the labels of the remaining nodes $\overline{\mathcal{L}} = \mathcal{V} \setminus \mathcal{L}$.

\noindent \textbf{Bayesian GCNs - Background and Extension:} A Bayesian framework for GCNs~\cite{zhang2019} (BGCN) views
the graph $\mathcal{G}$ and the GCN weights $\mathbf{W}$ as random quantities
and performs prediction by computing an expectation with respect to the
posterior distribution. In~\cite{zhang2019}, $\mathcal{G}_{obs}$ is viewed as a sample
realization from a {\em parametric} random graph model; the goal is the posterior inference of $p(\mathcal{G}|\mathcal{G}_{obs})$
marginalizing with respect to the random graph parameters. However, this
approach ignores any possible dependence of the graph $\mathcal{G}$ on
the features $\BX$ and the labels $\mathbf{Y_{\mathcal{L}}}$~\cite{ma2019}. 
Although subsequent works address this issue, they either use variational approximation~\cite{elinas2020} of the posterior, or rely on {\em maximum a posteriori} (MAP) estimation~\cite{pal2020}. Neither approach attempts to generate graph samples from the actual posterior. 

We propose a modified version of the BGCN. 
In the proposed BGCN, the conditional distribution of the
graph $\mathcal{G}$ is represented as
$p(\mathcal{G}|\mathcal{G}_{obs}, \BX, \mathbf{Y_{\mathcal{L}}})$.
This graph distribution can incorporate the information provided by the features $\BX$ and the training labels $\mathbf{Y_{\mathcal{L}}}$.
In the classification task, the posterior is defined over the GCN weights $\mathbf{W}$.  As in~\cite{gal2016}, we model the prior $p(\mathbf{W}) = \mathcal{N}(\mathbf{W}; \mathbf{0}, \mathbf{I})$. The likelihood $p(\mathbf{Y}_{\mathcal{L}}|\BX,\mathcal{G}, \mathbf{W})$ is modelled using a $K$-dimensional categorical distribution by applying softmax function at the last layer of the Bayesian GCN over the graph $\mathcal{G}$. The posterior distribution of $\mathbf{W}$ can be written as:
\begin{align}
p(\mathbf{W}|\BX, \mathbf{Y_{\mathcal{L}}},\mathcal{G}) &\propto
p(\mathbf{W})p(\mathbf{Y_{\mathcal{L}}}|\mathbf{W},\BX,\mathcal{G})\,.
\end{align}

For predicting labels for the unlabelled nodes, we only need to draw samples from $p(\mathbf{W}|\BX, \mathbf{Y_{\mathcal{L}}},\mathcal{G})$, rather than explicitly evaluate the probability, so the normalization constant  $1/p(\mathbf{Y}_{\mathcal{L}}|\BX, \mathcal{G})$ can be ignored. The posterior over $\mathbf{W}$ induces a predictive marginal posterior for the unknown node labels $\mathbf{Z}$ as follows:
\begin{align}
p(\BZ|\BX,\mathbf{Y_{\mathcal{L}}},\mathcal{G}_{obs}) 
=  \int  p(\BZ|\mathbf{W},\mathcal{G}_{obs},\BX) p(\mathbf{W}|\BX,\mathbf{Y_{\mathcal{L}}},\mathcal{G}) p(\mathcal{G}|\mathcal{G}_{obs}, \BX,\mathbf{Y_{\mathcal{L}}}) \,d\mathbf{W}\,d\mathcal{G} \label{eq:exact_posterior}\,.
\end{align}
The integral in equation~\eqref{eq:exact_posterior} cannot be computed in a closed form. Hence, Monte Carlo sampling is used for approximation:
\begin{align}
&p(\BZ|\BX,\mathbf{Y_{\mathcal{L}}},\mathcal{G}_{obs}) \approx \dfrac{1}{N_G S}\sum_{i=1}^{N_G}\sum_{s=1}^S p(\BZ|\mathbf{W}_{s,i},\mathcal{G}_{obs},\BX)\,.
\label{eq:MC_posterior}
\end{align}
Here, $N_G$ graphs $\mathcal{G}_{i}$ are
sampled from $p(\mathcal{G}|\mathcal{G}_{obs}, \BX,\mathbf{Y_{\mathcal{L}}})$ 
and subsequently for each $\mathcal{G}_{i}$,  $S$ weight matrices
$\mathbf{W}_{s,i}$ are sampled from the variational approximation of
$p(\mathbf{W}|\BX,\mathbf{Y_{\mathcal{L}}},\mathcal{G}_{i})$. 
Due to its simplicity, we use Monte Carlo dropout~\cite{gal2016} to generate the samples $\mathbf{W}_{s,i}$ from the variational approximation, although its generalization for GCNs in~\cite{hasanzadeh2020} could also be used.


\noindent \textbf{Node copying for BGCN:} The node copying model provides a way to specify  $p(\mathcal{G}|\mathcal{G}_{obs},\BX, \mathbf{Y_{\mathcal{L}}})$ and to generate samples from it. The proposed BGCN algorithm is presented in Algorithm 1. The inputs are the observed graph, the node features and the observed labels. The outputs are soft predictions for the unknown node labels. For the copying model, we need to specify    
$p(\zeta^j = m|\mathcal{G}_{obs}, \BX, \mathbf{Y_{\mathcal{L}}})$, because this defines the graph distribution.  In this setting, intuitively we believe that nodes with the \textit{same labels} are \textit{similar} and hence should be candidates for copying. For constructing the conditional distribution of $\boldsymbol{\zeta}$, we
employ a base classification algorithm using the observed graph
$\mathcal{G}_{obs}$, the features $\BX$ and the training labels
$\mathbf{Y_{\mathcal{L}}}$ to obtain predictive labels
$\hat{c}_{\ell} \in \{1,..., K\}$ for each node $\ell$ in the
graph (Step 3: Initialization). Then, for each node $j$, we assign a uniform probability for
$\zeta^j$ over all nodes with the same predictive label from the base
classifier as follows:
\begin{align}
p(\zeta^j = m|\mathcal{G}_{obs}, \BX, \mathbf{Y_{\mathcal{L}}}) = \begin{dcases} \frac{1}{|\mathcal{C}_k|}, \enspace \text{if } \hat{c}_j = \hat{c}_m = k\,, \text{ where }\mathcal{C}_k = \{\ell \mid \hat{c}_{\ell} = k\} \\
0, \enspace  \text{otherwise}\,.\end{dcases}
\label{eq:nc_zeta}
\end{align}
In our experiments, we use the GCN~\cite{kipf2017} as a base classifier to obtain $\{\hat{c}_{\ell}\}$.

\begin{algorithm}[ht]
\caption{Bayesian GCN with node copying}
\label{alg:bgcn_copying}
\begin{algorithmic}[1]
\STATE {\bfseries Input:}  $\mathcal{G}_{obs}$, $\BX$, $\mathbf{Y_{\mathcal{L}}}$
\STATE {\bfseries Output:}  $p(\BZ|\BX,\mathbf{Y_{\mathcal{L}}},\mathcal{G}_{obs})$
\STATE {\bfseries Initialization:} Train a base classifier to obtain $\hat{c}_{\ell}$ and form $\mathcal{C}_k$, $k = 1, 2, ... , K$.
\FOR{$i=1$ {\bfseries to} $N_G$}
\STATE  Sample graph $\mathcal{G}_{i} \sim p(\mathcal{G}|\mathcal{G}_{obs}, \BX, \mathbf{Y_{\mathcal{L}}})$ from node copying model defined by~\eqref{eq:nc_zeta} and~\eqref{eqn:model}.
\FOR{$s=1$ {\bfseries to} $S$}
\STATE Sample weights $\mathbf{W}_{s,i}$ using MC dropout by training a GCN over graph $\mathcal{G}_{i}$.
\ENDFOR
\ENDFOR
\STATE Approximate $p(\BZ|\BX,\mathbf{Y_{\mathcal{L}}},\mathcal{G}_{obs})$ using eq.~\eqref{eq:MC_posterior}.
\end{algorithmic}
\end{algorithm}

\noindent \textbf{Experiments:} We evaluate node classification on standard benchmark datasets, including citation datasets Cora, Citeseer and Pubmed from~\cite{sen2008} and Coauthor-CS from~\cite{shchur2018}. For the citation networks, we use the same hyperparameters as in~\cite{kipf2016} and for the Coauthor-CS dataset we use the largest connected component and the hyperparameters reported in~\cite{shchur2018}. We address three different scenarios, where we have access to 5, 10 or 20 labels per class. We conduct 50 trials; each trial corresponds to a random weight initialization and use of a random train/test split.

\begin{table}[htbp!]
\centering
\caption{Accuracy of semi-supervised node classification.}
\label{table:bgcnresult}
\setlength{\tabcolsep}{4pt}
\resizebox{\columnwidth}{!}{
\begin{tabular}{l|llll|llll}
\toprule 
\textbf{Algorithms} &  &\textbf{5 labels}        &\textbf{10 labels}         &\textbf{20 labels} &&\textbf{5 labels}        &\textbf{10 labels}         &\textbf{20 labels} \\ \midrule[0.25ex]
\textbf{GCN}    &{\multirow{5}{*}{\rotatebox[origin=c]{90}{Cora}}}         &70.0$\pm$3.7            &76.0$\pm$2.2     
&79.8$\pm$1.8  &{\multirow{5}{*}{\rotatebox[origin=c]{90}{Citeseer}}}  &58.5$\pm$4.7            &65.4$\pm$2.6              &67.8$\pm$2.3  \\
\textbf{MMSBM-BGCN}    &&\textbf{74.6$\pm$2.8}*   &77.5$\pm$2.6   &80.2$\pm$1.5  &  &63.0$\pm$4.8   &\textbf{69.9$\pm$2.3}*     &\textbf{71.1$\pm$1.8}*  \\
\textbf{DFNET-ATT}        &        &72.3$\pm$2.9&75.8 $\pm$1.7 &79.3$\pm$1.8 &  &60.5$\pm$1.2&63.2 $\pm$2.9 &66.3$\pm$1.7 \\
\textbf{SBM-GCN}         &  &46.0$\pm$19        &74.4$\pm$10   &\textbf{82.6$\pm$0.2}*  & 
&24.5$\pm$7.3    &43.3$\pm$12  &66.1$\pm$5.7           \\
\textbf{BGCN-Copy}  &  &73.8$\pm$2.7   &\textbf{77.6$\pm$2.6}  &80.3$\pm$1.6 & &\textbf{63.9$\pm$4.2}*   &68.5$\pm$2.3     &70.2$\pm$2.0  \\ \midrule[0.25ex]
\textbf{GCN}      &{\multirow{4}{*}{\rotatebox[origin=c]{90}{Pubmed}}}      &69.7$\pm$4.5           &73.9$\pm$3.4    &77.5$\pm$2.5   &{\multirow{4}{*}{\rotatebox[origin=c]{90}{Co.-CS}}} &90.7$\pm$1.4   &90.7$\pm$1.4 &\textbf{92.1$\pm$1.0 } \\
\textbf{MMSBM-BGCN}  &  &70.2$\pm$4.5  &73.3$\pm$3.1             &76.0$\pm$2.6  &&\textbf{91.0$\pm$1.0}           &\textbf{91.3$\pm$1.1}*    &91.6$\pm$1.0    \\
\textbf{SBM-GCN}       &    &59.0$\pm$10        &67.8$\pm$6.9  &74.6$\pm$4.5    &  &88.8$\pm$1.7           &89.8$\pm$1.7  &91.4$\pm$1.8  \\
\textbf{BGCN-Copy}   &&\textbf{71.0$\pm$4.2}*  &\textbf{74.6$\pm$3.3}*             &\textbf{77.5$\pm$2.4}    &  &90.5$\pm$1.4      &90.7$\pm$1.2   &91.6$\pm$1.0       \\ \bottomrule
\end{tabular}}
\label{tab:my_label}
\end{table}

We compare the proposed BGCN based on node copying (BGCN-Copy) with the BGCN based
on an MMSBM~\cite{zhang2019} (MMSBM-BGCN) and the SBM-GCN~\cite{ma2019} to highlight the usefulness of the
node copying model. We also include node classification baselines
GCN~\cite{kipf2017} and DFNET\cite{wijesinghe2019} (only for Cora and Citeseer datasets due to run-time considerations). The
average accuracies along with standard errors are reported in Table~\ref{table:bgcnresult}. For each setting, the highest average accuracy is written in bold and an asterisk (*) is used if the algorithm offers better performance compared to the second best algorithm that is statistically significant at the 5\% level using a Wilcoxon signed rank test. (The significance test results are reported in the same way in the two following sections).

\textbf{Data scarce setting}: As in~\cite{ma2019}, we consider another node classification experiment where all the edges connected to the test set nodes are removed from the graph. We also have access to only 5 training examples from each class. We use a Multi Layer Perceptron (MLP) as a non-graph baseline algorithm. We conduct 20 trials for each dataset and report the average accuracies along with the standard error in Table~\ref{tab:data_scarce}.

\begin{table}[htbp!]
\centering
\caption{Accuracy of semi-supervised node classification in data-scarce setting.}
\setlength{\tabcolsep}{4pt}
\begin{tabular}{lccccc}
\toprule
&\textbf{MLP} &\textbf{GCN} &\textbf{MMSBM-BGCN} &\textbf{SBM-GCN} &\textbf{BGCN-Copy} \\ \midrule
\textbf{Cora}        &39.7$\pm$3.7                                      &53.5$\pm$3.6                                       &54.7$\pm$4.6                                             &25.3$\pm$12.8                                         &\textbf{58.7$\pm$3.5}*                 \\ 
\textbf{Citeseer}  &40.2$\pm$3.6                                      &48.3$\pm$3.3                                       &47.6$\pm$2.5                                             &18.1$\pm$5.4                                          &\textbf{54.5$\pm$2.8}*                 \\ 
\textbf{Pubmed}      &59.2$\pm$3.4                                      &66.2$\pm$4.1                                       &67.1$\pm$3.9                                             &57.8$\pm$7.7                                          &\textbf{69.1$\pm$4.0}*                 \\ 
\textbf{Coauthor CS} &80.5$\pm$2.5                                      &85.8$\pm$2.3                                       &88.2$\pm$1.2                                             &\textbf{88.7$\pm$1.8    }          &87.4$\pm$1.6                                            \\  \bottomrule
\end{tabular}
\label{tab:data_scarce}
\end{table}
\noindent \textbf{Comparison of BGCN generative models with runtime analysis}:
Finally, we compare the node copying BGCN with BGCN variants that incorporate other graph generative
models for $p(\mathcal{G} | \mathcal{G}_{obs}, \mathcal{D})$, including the generative models considered in Section~\ref{sec:analysis}. Performance and runtime results on the Cora dataset are reported in Table~\ref{table:bgcn_cora}. Implementation details of the BGCN variants are provided in~\ref{ap:bgcndet}.

\begin{table}[htbp]
\centering
\caption{Accuracy of node classification for different size of labeled set and runtime (for a 20 labels per class trial) of various BGCNs on Cora dataset, averaged over 20 trials.}
\setlength{\tabcolsep}{4pt}
\resizebox{\columnwidth}{!}{
\begin{tabular}{lcccccc}
\toprule
\textbf{Gen. model} &\textbf{VGAE}  &\textbf{GRAPHITE} &\textbf{MMSBM} &\textbf{NETGAN} &\textbf{DGLFRM-B} &\textbf{BGCN-Copy}\\ \midrule
\textbf{5 labels} &68.3$\pm$3.6                                                   &69.0$\pm$3.1                                           &74.6$\pm$2.8                                        &71.0$\pm$2.4                                         &\textbf{75.1$\pm$2.3  }      &73.8$\pm$2.7                                            \\  
\textbf{10 labels} &74.0$\pm$2.0                                                   &74.3$\pm$2.0                                           &77.5$\pm$2.6                                        &77.0$\pm$2.7                                         &77.3$\pm$1.8                                           &\textbf{77.6$\pm$2.6}          \\ 
\textbf{20 labels}  &78.2$\pm$2.0                                                   &78.3$\pm$1.7                                           &80.2$\pm$1.5                                        &80.0$\pm$1.62                                        &78.8$\pm$1.6                                           &\textbf{80.3$\pm$1.6   }               \\  \midrule
\textbf{Exec. time (s)} &33                                                             &34                                                     &512                                                 &86573                                                &244                                                    &\textbf{30 }                                                     \\  \bottomrule
\end{tabular}}
\label{table:bgcn_cora}
\end{table}
\noindent \textbf{Discussion}: From Table~\ref{table:bgcnresult}, we observe that the proposed BGCN either outperforms or offers comparable performance to the competing techniques in most cases. In the data-scarce setting, models which can incorporate the uncertainty in the graph show better performance in general as they have some capability to mitigate the effects of missing edges and fewer training labels. The proposed BGCN-Copy has the capacity of adding new edges during training, as do MMSBM-BGCN and SBM-GCN. From Table~\ref{tab:data_scarce}, we observe that the proposed algorithm shows significant improvement compared to its competitors.

Table~\ref{table:bgcn_cora} highlights that the node copying
based BGCN provides much better results than VGAE or GRAPHITE and is
much faster compared to the other alternatives. 
The proposed copying approach generalizes the BGCN to settings where the parametric model does not fit well. 
It also reveals that our approach is faster than MMSBM-BGCN which makes it a better alternative to process larger datasets. The parametric inference for the MMSBM is challenging for larger graphs, whereas our approach scales linearly with the number of nodes.  

\section{Application -- Node copying for defense against adversarial attacks}
We present node classification in an adversarial setting as a second application. Generally, a topological attack corrupts the prediction of a node label by tampering with its neighborhood. Since the node copying model explicitly provides new neighborhoods to nodes, we can construct a defense algorithm based on sampled graphs in which the targeted node can possibly escape the effect of the attack.

\noindent \textbf{Problem Setting}:  We consider attacks in which unlabeled nodes $\mathcal{V}_{attacked} \subset \overline{\mathcal{L}}$ are subjected to an adversarial attack that can modify the neighborhood of the targeted nodes. The defense algorithm only has access to the resulting corrupted graph $\mathcal{G}_{attacked}$, and the goal is to recover classification accuracy at the targeted nodes $v \in \mathcal{V}_{attacked}$. The complete feature matrix $\BX$ and the training labels $\mathbf{Y_{\mathcal{L}}}$ of a (small) labeled set are assumed to be unperturbed. Following the  categorizations of graph attacks in the survey~\cite{lichao2018}, this specific adversarial setting belongs to the {\em edge-level, targeted, poisoning} category.

\noindent \textbf{Defense algorithm using node copying}:
 Our proposed defense procedure is summarized in Algorithm~\ref{alg:copying_attack}. We now detail the steps of the algorithm. We first train a GCN classification algorithm on the attacked graph (Step 3). Since our sole interest is in correcting the erroneous classifications at the nodes in $\mathcal{V}_{attacked}$, we do not need to sample entire graphs for this application. Instead, we
sample local neighbourhoods of the targeted nodes
$v \in \mathcal{V}_{attacked}$ via the node copying model
$p(\zeta^v| \mathcal{G}_{attacked}, \BX)$. We note that the classifications of the
attacked nodes are most likely incorrect. So, using the predicted classes to define similarities in the node copying model is not sensible as the targeted nodes might be more similar to nodes from a different class. Instead, we use an unsupervised representation of the nodes for sampling the replacement nodes. We train a node embedding algorithm using $(\mathcal{G}_{attacked}, \BX)$ to obtain the node representations $\{e_i\}_{i=1:N}$ (Step 4). A symmetric, pairwise distance matrix $D \in \real_{+}^{N \times N}$ is formed, where $D_{ij}$ denotes a distance between $e_i$ and $e_j$. (Step 5).  In our experiments, we use the Graph Variational Auto Encoder (VGAE)~\cite{kipf2017} as the embedding algorithm and compute pairwise Euclidean distances to form $D$ : $D_{i,j} = ||e_i-e_j ||_2$. For any targeted node, we sample the replacement node (Step 8) uniformly at random from the nodes which are close to the targeted node. Formally, we define:
\begin{align}\label{eq:app:zeta_dist}
p(\zeta^v = m|\mathcal{G}_{attacked}, \BX, \mathbf{Y_{\mathcal{L}}}) =\begin{dcases} \frac{1}{P}, \enspace \text{if } D_{v,m} = D_{v,(\ell)} \text{ for some } 1 \leq \ell \leq P\,\\
0, \enspace  \text{otherwise}.\end{dcases} 
\end{align}
Here $D_{i,(j)}$ is the $j$-th order statistic of $\{D_{i,\ell}\}_{\ell=1}^N$. For each corrupted
node, our approach replaces the classification obtained from the node classifier trained on the corrupted graph $\mathcal{G}_{attacked}$ by the ensemble of classifications of the same model with the sampled $\zeta^v$s copied in place of node $v$ (Step 9). We use the mean of the softmax probabilities as the ensemble aggregation function (Step 11).  In many graph-based models~\cite{kipf2016,velivckovic2018}, the node classification is primarily influenced by nodes within a few hops. As a result, a localized evaluation of the predictions at the targeted nodes can be computed efficiently.

\begin{algorithm}[htbp]
\caption{Error correction using node copying}
\label{alg:copying_attack}
\begin{algorithmic}[1]
\STATE {\bfseries Input:}  $\mathcal{G}_{attacked}$, $\BX$, $\mathbf{Y_{\mathcal{L}}}$, $\mathcal{V}_{attacked}$
\vspace{0.1cm}
\STATE {\bfseries Output:} $\widehat{\mathbf{Y}}_{attacked}^{Copying}$
\vspace{0.1cm}
\STATE Train a semi supervised node classification algorithm using $\mathcal{G}_{attacked}, \BX, $ $\mathbf{Y_{\mathcal{L}}}$ to learn model parameters $\textbf{W}$.

\STATE Train a node embedding algorithm using $\mathcal{G}_{attacked}, \BX$  to obtain embeddings $\{e_i\}_{i=1}^N$. 

\STATE Compute the pairwise distance matrix $D$, where, $D_{ij} = \lvert \lvert e_i -e_j\rvert \rvert_2$.

\FOR{$v \in \mathcal{V}_{attacked}$}
\FOR{$k = 1:N_G$}
\STATE Sample $\zeta^v_k \sim p(\zeta^v |\mathcal{G}_{attacked}, \BX, \mathbf{Y_{\mathcal{L}}})$ according to~\eqref{eq:app:zeta_dist}.

\STATE Copy node $\zeta^v_k$ in place of node $v$ and compute the prediction of the learned classifier (in Step 3) at node $v$,  $\bm{\hat{y}}_{v}^{(\zeta^v_k)}$ using the parameters $\mathbf{W}$.
\ENDFOR
\STATE Compute $\bm{\hat{y}}_{v}^{Copying} = \frac{1}{N_G}\sum_{k=1}^{N_G} \bm{\hat{y}}_{v}^{(\zeta^v_k)}$
\ENDFOR
\STATE  Form $\widehat{\mathbf{Y}}_{attacked}^{Copying} = \{\bm{\hat{y}}_{v}^{Copying}\}_{v \in \mathcal{V}_{attacked}}$
\end{algorithmic}
\end{algorithm}
We note that the node embeddings $\{e_i\}^N_{i=1}$ of $\mathcal{G}_{attacked}$, which are used to form  $p(\zeta^v |\mathcal{G}_{attacked}, \BX, \mathbf{Y_{\mathcal{L}}})$ are also affected by the topological attack. This can potentially degrade the effectiveness of the proposed defense mechanism. If the attack is believed to be too strong, we can use a similarity metric that ignores any topological information instead of relying on the embeddings of the nodes of the attacked graph. However, our experimental results in Table~\ref{tab:attack_result} demonstrates that even for a quite severe attack (where $75\%$ of the neighbors of the targeted nodes have been tampered with), the proposed defense strategy using node copying has impressive performance.
\begin{figure}[ht]
\centering
\includegraphics[trim={0 0 0 1em}, scale=0.325, clip]{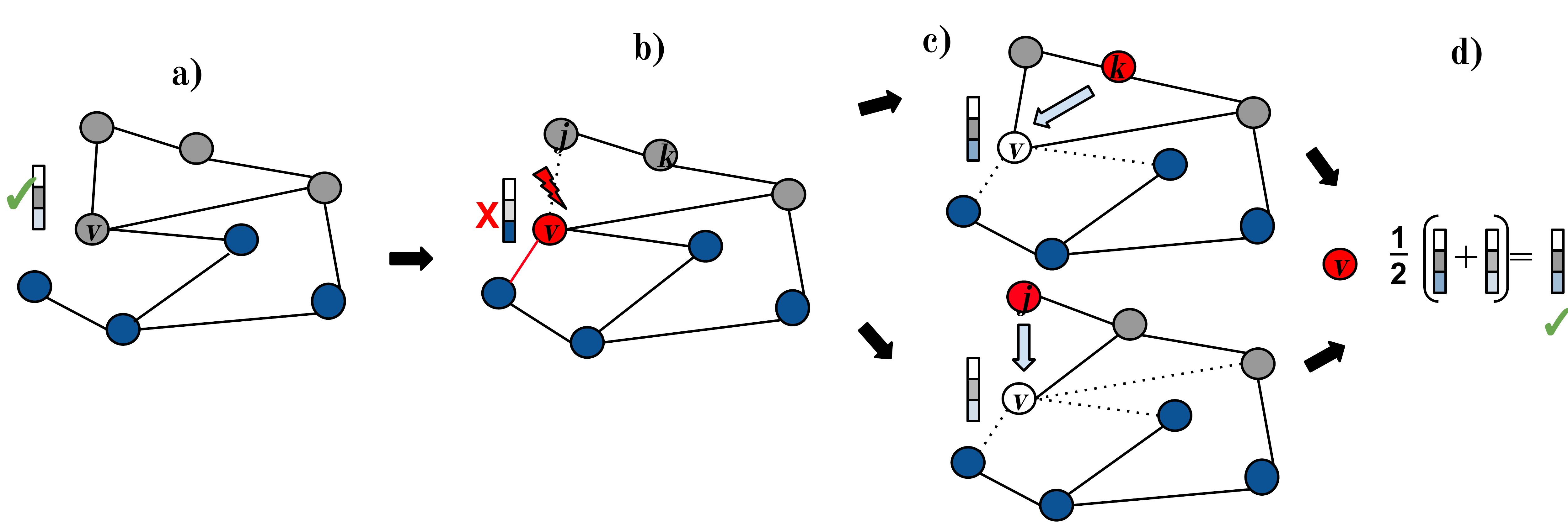}
\caption{Summary of the node copying procedure. \textbf{a)} In the absence of the attack, the softmax of node $v$ achieves the correct classification in $\mathcal{G}_{obs}$. \textbf{b)} Node $v$ is targeted by an attack and is now wrongly classified. \textbf{c)} We sample two replacement nodes for node $v$ : $\zeta^v=j$ and $\zeta^v=k$. The softmax vectors at node $v$ after nodes $j$ and $k$ are copied in place of node $v$ are denoted as  $\bm{\hat{y}}^{(j)}_{v}$ and $\bm{\hat{y}}^{(k)}_{v}$, respectively. \textbf{d)} The error for node $v$ is corrected by computing $\bm{\hat{y}}_{v}^{Copying} = \frac{1}{2}\big(\bm{\hat{y}}^{(j)}_{v} + \bm{\hat{y}}^{(k)}_{v} \big)$.}
\label{fig:node_copying}
\end{figure} 
Figure~\ref{fig:node_copying}(b) shows how a topological attack can lead to an incorrect classification for a targeted node. Figure~\ref{fig:node_copying}(c) provides two examples of the node copying operation to generate new graphs by copying some $\zeta^v$s in place of node $v$. Figure~\ref{fig:node_copying}(d) depicts how the softmax values derived from these generated graphs are combined to recover the original correct classification.

\noindent \textbf{Experiments:}~\label{sec:exp:att:dice}
Our proposed defense algorithm is evaluated against three graph adversarial attacks. {\em Targeted DICE} (Delete Internally Connect Externally) is the targeted version of the global attack from~\cite{waniek2018Dice}. This algorithm attacks a node by randomly disconnecting $\beta$ percent of its neighbors of the same class, and add the same number of edges to random nodes of a different class.  We present results with $\beta = 50\%$ and $\beta =75\%$. {\em Nettack} from~\cite{zugner2018} is another topological attack. We use the directed version of {\em Nettack} and set the number of perturbations to be equal to the degree of the attacked node. The last attack we consider is {\em FGA}~\cite{chen2018d}, for which we set the number of perturbations to 20. All remaining parameters are set to the values provided in the respective papers.

In order to illustrate the effectiveness of the procedure, we compare the accuracy at the attacked nodes for the proposed copying algorithm with a standard GCN and a {\em state-of-the-art} defense algorithm GCN-SVD~\cite{entezari2020}, for which we set the rank parameter $K=50$. We consider a setting where $\mathcal{V}_{trn}$ is formed with 20 labels per class. The attacked set $\mathcal{V}_{attacked}$ is simulated by randomly sampling 40 nodes, excluding $\mathcal{V}_{trn}$, and corrupting them with the attack.

\begin{table}[htbp]
\centering
\caption{Accuracy of defense algorithms at the attacked node averaged over 50 trials.} \label{tab:attack_result}
\setlength{\tabcolsep}{3pt}
\resizebox{\columnwidth}{!}{
\begin{tabular}{lcccc|cccc}
\toprule
\multicolumn{1}{l}{} &\multicolumn{4}{c}{\textbf{Cora}} &\multicolumn{4}{c}{\textbf{Citeseer}} \\ 
\midrule
\textbf{} &\textbf{DICE $50\%$} &\textbf{DICE $75\%$} &\textbf{Nettack} &\textbf{FGA} &\textbf{DICE $50\%$} &\textbf{DICE $75\%$} &\textbf{Nettack} &\textbf{FGA} \\ \midrule
\textbf{GCN} &53.4$\pm$8.9 &31.9$\pm$8.9 &13.6$\pm$9.6 &69.8$\pm$16 &45.5$\pm$7.9 &30.45$\pm$7.6 &11.8$\pm$5.1 &54.6$\pm$15 \\
\textbf{GCN-SVD} &51.1$\pm$9.8 &35.9$\pm$7.1 &\textbf{41.7$\pm$11 }&57.2$\pm$15 &48.4$\pm$7.7 &36.3$\pm$9.1 &\textbf{34.4$\pm$8.9 *}&39.4$\pm$17 \\
\textbf{Copying} &\textbf{57.5$\pm$9.9*} &\textbf{38.3$\pm$7.4*} &40.5$\pm$11 &\textbf{70.2$\pm$16 }&\textbf{50.1$\pm$9.1*} &\textbf{37.6$\pm$9.1} &31.6$\pm$7.7 &\textbf{55.8$\pm$12} \\ \bottomrule
\end{tabular}}
\end{table}

\noindent \textbf{Discussion}: From Table~\ref{tab:attack_result}, we see that the proposed
corrective mechanism based on node copying improves the classification
accuracy of the corrupted nodes. Interestingly, the proposed correction procedure does not involve
sampling of the entire graph or retraining of the model, rather we can perform computationally inexpensive, localized computations to improve accuracy at the targeted nodes. We observe that the  GCN-SVD approach outperforms the Copying algorithm for {\em Nettack}, which is expected as it is tailored to provide efficient defense for this specific attack, but performs significantly worse than a standard GCN against the {\em FGA} attack. In contrast, the proposed  Copying defense offers consistent improvement over  GCN across different attacks. 

\section{Application -- Node copying for recommender systems}\label{sec:app_rec}
As the last application, we consider the task of graph-based personalized item recommendation. An intuitive solution to this problem is to form the item recommendation of a user based on other users that have many items in common. If two users share multiple items in their interaction histories, we presume that one item purchased/liked by one of them could be recommended to the other. A node copying model can directly apply this approach by using a similarity metric between users that is proportional with the number of shared items. \\
\noindent \textbf{Problem Setting}:
Let $\mathcal{U}$ be the set of users and $\mathcal{I}$ be the set of items. $\mathcal{G}_{obs}$ is the partially observed bipartite graph
built from previous user-item interactions. The task is to infer
other unobserved interactions, which can be seen as a link prediction
problem~\cite{Wang2019NeuralGC,Ying2018GraphCN}. Alternatively, we can view this task as a ranking
problem~\cite{rendle2009}. For each user $u$, for an observed interaction
with item $i$ and a non-observed interaction with item $j$, we have
$i >_u j$ in the training set. If both $(u,i)$ and $(u,j)$ are
observed, we do not have a ranking. Using this interaction data
$\{>_u\}_{trn} = \{(u,i,j) : (u,i) \in \mathcal{G}_{obs}, (u,j) \notin
\mathcal{G}_{obs}\}$ for fitting a model, the task is to rank all
$(u,i,j)$ such that both $(u,i)$ and $(u,j) \notin \mathcal{G}_{obs}$. The sets of pairs
$\{(u,i) : (u,i) \in \mathcal{G}_{obs}\}$ and
$\{(u,j) : (u,j) \notin \mathcal{G}_{obs}\}$  are referred to as
the positive and negative pools of interactions respectively. We
denote the test set as
$\{>_u\}_{test} = \{(u,i,j) : (u,i) \notin \mathcal{G}_{obs}, (u,j)
\notin \mathcal{G}_{obs}\}$. We consider a Bayesian Personalized
Recommendation (BPR) framework~\cite{rendle2009} (details in~\ref{appendix:bpr}) along with the inference of the graph $\mathcal{G}$.

\noindent \textbf{Background}: Many existing graph-based deep learning recommender system models~\cite{Sun2019MultigraphCC,Wang2019NeuralGC,Ying2018GraphCN}
learn weights $\mathbf{W}$ to form an embedding $e_u(\mathbf{W}, \mathcal{G}_{obs})$ for
the $u$-th user and $e_i(\mathbf{W}, \mathcal{G}_{obs})$ for the $i$-th item
node. The ranking of any $(u,i,j)$ triple such that both $(u,i) \notin \mathcal{G}_{obs} \text{ and } (u,j) \notin \mathcal{G}_{obs}$ is specified as $p(i >_u j|\mathcal{G}_{obs}, \mathbf{W}) = \sigma(e_u \cdot e_i - e_u \cdot e_j)\,,$
where $\sigma(\cdot)$ is the sigmoid function and $\cdot$ denotes the inner product. Assuming a suitable prior on $\mathbf{W}$, the learning goal is to maximize the posterior of $\mathbf{W}$ on the training set. Here,
$\widehat{\mathbf{W}} = \displaystyle{\argmax_{\mathbf{W}}}\Hquad
p(\mathbf{W}|\{>_u\}_{trn},\mathcal{G}_{obs})$ is learned by maximizing the BPR
objective using Stochastic Gradient Descent (SGD). The pairwise ranking probabilities in the test set is then computed using $\widehat{\mathbf{W}}$.

\noindent \textbf{Ensemble BPR}: We summarize the proposed Ensemble BPR algorithm in Algorithm~\ref{alg:ebpr}. 
For the Ensemble BPR algorithm, we first obtain $\widehat{\mathbf{W}}$ as previously described (Step 3), then we evaluate the ranking by computing an expectation with respect to the random graph $\mathcal{G}$, sampled from the node copying model.
\begin{algorithm}[ht]
\caption{Ensemble BPR with node copying}
\label{alg:ebpr}
\begin{algorithmic}[1]
\STATE {\bfseries Input:}  $\mathcal{G}_{obs}$, $\{>_u\}_{trn}$
\STATE {\bfseries Output:}  $p(\{>_u\}_{test}|\{>_u\}_{trn},\mathcal{G}_{obs})$
\STATE  Obtain $\widehat{\mathbf{W}} = \displaystyle{\argmax_\mathbf{W}}\Hquad
p(\mathbf{W}|\{>_u\}_{trn},\mathcal{G}_{obs})$ by minimizing the BPR
loss.
\FOR{$i=1$ {\bfseries to} $N_G$}
\STATE Sample graph $\mathcal{G}_{i} \sim p(\mathcal{G}|\mathcal{G}_{obs}, \{>_u\}_{trn})$ using the node copying model defined by~\eqref{eq:zeta_rec}.
\ENDFOR
\STATE Approximate $p(\{>_u\}_{test}|\{>_u\}_{trn},\mathcal{G}_{obs})$ using eq.~\eqref{eq:mc_immediate_copy}.
\end{algorithmic}
\end{algorithm}

In this setting, we assign a non-zero conditional probability to only the class of bipartite graphs since $\mathcal{G}_{obs}$ is bipartite. This is achieved by considering a node copying scheme where only user nodes can be copied to a user node. We consider that users that share items are similar, so we use the Jaccard index $\rho$ between the sets of items for pairs of users to define the conditional probability distribution of $\boldsymbol{\zeta}$ as follows:
\begin{align}
p(\zeta^j=m|\mathcal{G}_{obs})= \begin{dcases}
\rho(j,m)/ \displaystyle{\sum_{i \in \mathcal{U}}}\rho(j, i) \,,  &\text{ if } j, m \in \mathcal{U}  \\
0 \,, &\text{ otherwise} 
\end{dcases}
\label{eq:zeta_rec}
\end{align}
We compute the pairwise ranking probabilities in the test set as follows:
\begin{align}
p(\{>_u\}_{test}|\{>_u\}_{trn},\mathcal{G}_{obs}) = \int p(\{>_u\}_{test}|\mathcal{G}, \mathbf{W}) p(\mathbf{W}|\{>_u\}_{trn},\mathcal{G}_{obs}) p(\mathcal{G}| \mathcal{G}_{obs}, \{>_u\}_{trn}) d\mathbf{W} d\mathcal{G} \,.\label{eq:int_immediate_copy}
\end{align}
We sample $N_G$ graphs $\mathcal{G}_{i}$s from $p(\mathcal{G}|\mathcal{G}_{obs}, \{>_u\}_{trn})$ using node copying (Step 5) and then form a Monte Carlo approximation of eq.~\eqref{eq:int_immediate_copy} as follows (Step 7):
\begin{align}
&p(\{>_u\}_{test}|\{>_u\}_{trn},\mathcal{G}_{obs}) \approx \frac{1}{N_G}\sum_{i=1}^{N_G} p(\{>_u\}_{test}|\mathcal{G}_{i}, \widehat{\mathbf{W}})\,,\label{eq:mc_immediate_copy}
\end{align}
Implementation of eq.~\eqref{eq:mc_immediate_copy} does not require retraining of the model since $\widehat{\mathbf{W}}$ is already obtained by minimizing the BPR loss; it only involves evaluation of a trained model on multiple sampled graphs to compute the average pairwise ranking probability for the test set.

\noindent \textbf{Sampled Graph BPR (SGBPR) --- training with sampled graphs:}

We consider another approach where we use the generated graphs during the
training process as well. The SGBPR algorithm is summarized in Algorithm~\ref{alg:sgbpr}.
\begin{algorithm}[ht]
\caption{Sampled Graph BPR with node copying}
\label{alg:sgbpr}
\begin{algorithmic}[1]
\STATE {\bfseries Input:}  $\mathcal{G}_{obs}$, $\{>_u\}_{trn}$
\STATE {\bfseries Output:}  $p(\{>_u\}_{test}|\{>_u\}_{trn},\mathcal{G}_{obs})$
\FOR{$i=1$ {\bfseries to} $N_G$}
\STATE Sample graph $\mathcal{G}_{i} \sim p(\mathcal{G}|\mathcal{G}_{obs}, \{>_u\}_{trn})$ using the node copying model defined by~\eqref{eq:zeta_rec}.
\ENDFOR
\STATE Compute $\widehat{A}_{\overline{\mathcal{G}}}$  using eq.~\eqref{eq:A_hat_G_bar} and $\widehat{\overline{\mathcal{G}}}_b$ (equivalently $A_{\widehat{\overline{\mathcal{G}}}_b}$) using $\widehat{A}_{\overline{\mathcal{G}}}$.
\STATE Compute $\widehat{f(\mathcal{G}_{obs})} =
\mathcal{G}_{obs}\cup\widehat{\overline{\mathcal{G}}}_b$ 
\STATE  Obtain $\widehat{\mathbf{W}}' = \displaystyle{\argmax_\mathbf{W}} \Hquad p(\mathbf{W}|\{>_u\}_{trn},\mathcal{G}_{obs},\widehat{f(\mathcal{G}_{obs})})$ by minimizing the BPR
loss.
\STATE Approximate $p(\{>_u\}_{test}|\{>_u\}_{trn},\mathcal{G}_{obs})$ using eq.~\eqref{eq:mc_recommender}.
\end{algorithmic}
\end{algorithm}

Our motivation is to remove some of the
potentially unobserved positive interactions in the training set from
the negative interaction pool. We rely on the node copying model to
sample graphs (Step 4), which potentially contain positive interactions between
user-item pairs, which are unobserved in $\mathcal{G}_{obs}$. The inference of graph $\mathcal{G}$ is carried out from the
copying model $p(\mathcal{G}| \mathcal{G}_{obs}, \{>_u\}_{trn})$ specified in eq.~\eqref{eq:zeta_rec}. 
We need to compute:
\begin{align}
p(\{>_u\}_{test}|\{>_u\}_{trn},\mathcal{G}_{obs}) = \int &p(\{>_u\}_{test}|\mathcal{G}, \mathbf{W}) p(\mathbf{W}|\{>_u\}_{trn},\mathcal{G}_{obs},f(\mathcal{G}_{obs})) p(\mathcal{G}| \mathcal{G}_{obs}, \{>_u\}_{trn}) d\mathbf{W} d\mathcal{G} \,.\label{eq:int_recommender} 
\end{align}
Here, $f(\mathcal{G}_{obs})$ is a function of the observed graph that
returns a graph that is used to control the negative pool. There is flexibility in the choice of this function, but we use
$f(\mathcal{G}_{obs}) =
\mathcal{G}_{obs}\cup\overline{\mathcal{G}}_b$. The union $\cup$
indicates that we take the union of the edge sets of the two graphs.
We define $\overline{\mathcal{G}} \triangleq
\mathbb{E}_{\mathcal{G}|\mathcal{G}_{obs}}[\mathcal{G}]$ as the
graph with adjacency matrix equal to the expectation of the adjacency matrix over the
generative graph distribution. The graph $\overline{\mathcal{G}}_b$ is
derived from this; it has a binary adjacency matrix derived by comparing the
adjacency matrix entries of $\overline{\mathcal{G}}$ to a small
positive threshold $b$. So, we have $A_{\overline{\mathcal{G}}_b} = \mathbbm{1}[A_{\overline{\mathcal{G}}}>b]\,.$
Due to analytical intractability, we approximate
$\mathbb{E}_{\mathcal{G}|\mathcal{G}_{obs}}[\mathcal{G}]$ using Monte Carlo. This amounts to approximating the adjacency matrix of $\overline{\mathcal{G}} = \mathbb{E}_{\mathcal{G}|\mathcal{G}_{obs}}[\mathcal{G}]$, whose $(s, w)$-th entry can be estimated as:
\begin{align}
\widehat{A}_{\overline{\mathcal{G}}}(s,w) = \frac{1}{N_G}\sum_{i=1}^{N_G}A_{\mathcal{G}_{i}}(s,w) \,.\label{eq:A_hat_G_bar}
\end{align}
Here, $\mathcal{G}_{i} \sim p(\mathcal{G}|
\mathcal{G}_{obs}, \{>_u\}_{trn})$ is sampled from the node copying model in Step 4. This allows us to construct an estimate $\widehat{f(\mathcal{G}_{obs})}$ by first computing an approximate $\widehat{\overline{\mathcal{G}}}_b$ (equivalently $A_{\widehat{\overline{\mathcal{G}}}_b}$) using $\widehat{A}_{\overline{\mathcal{G}}}$ (Step 6) and then using this $\widehat{\overline{\mathcal{G}}}_b$ in the definition of $f(\mathcal{G}_{obs})$ (Step 7). 

In specifying $p(\mathbf{W}|\{>_u\}_{trn},\mathcal{G}_{obs},f(\mathcal{G}_{obs}))$,
we effectively aim to reduce the training set by eliminating
unobserved edges from the negative pool. We have $\{>_u\}_{trn} =
D_{S}$. We now introduce $D'_{S} = D_S\setminus \{(u,i,j): (u,i) \in \mathcal{G}_{obs}, (u,j)\in
\overline{\mathcal{G}}_b\}$, and set
$p(\mathbf{W}|\{>_u\}_{trn},\mathcal{G}_{obs},f(\mathcal{G}_{obs})) = p(\mathbf{W}|D'_{S})$, which amounts to training the model based on positive and negative interactions in $D'_{S}$. Using $\widehat{\mathbf{W}}' = \displaystyle{\argmax_\mathbf{W}} \Hquad p(\mathbf{W}|\{>_u\}_{trn},\mathcal{G}_{obs},\widehat{f(\mathcal{G}_{obs})})$ (Step 8), the posterior probability of $\{>_u\}_{test}$ in eq.~\eqref{eq:int_recommender} is then approximated as:
\begin{align}
p(\{>_u\}_{test}|\{>_u\}_{trn},\mathcal{G}_{obs}) &\approx \frac{1}{N_G}\sum_{i=1}^{N_G} p(\{>_u\}_{test}|\mathcal{G}_{i}, \widehat{\mathbf{W}}')\, \label{eq:mc_recommender}
\end{align}
\noindent \textbf{Experiments:}  For the recommender system, we use two datasets (ML100k and Amazon-CDs, with details in Table~\ref{table:rec_datasets} in~\ref{ap:data}) that we preprocess by retaining users and items that have a specified minimum number of edges set by a threshold. We create training, validation and test sets by splitting the data $70/10/20 \%$. We generate random splits for each of the 10 trials and report average Recall and Normalized Discounted Cumulative Gain (NDCG) at 10 and 20.  \textit{Recall@k} measures the proportion of the true (preferred) items from the top-$k$ recommendation. For a user $u\in\mathcal{U}$, the algorithm recommends an ordered set (in descending order of preference) of top-$k$ items $ I_k(u)=  \{i_n\}^k_{n=1}  \subset \mathcal{I}$. There is a set of true preferred items for a user $\mathcal{I}_u^+$ and the number of true positive is $|\mathcal{I}_u^+\cap I_k(u)|$, so the recall is computed as follows: $\mathrm{Recall@k}=\frac{|\mathcal{I}_u^+\cap I_k(u)|}{|\mathcal{I}_u^+|}$. Normalized Discounted Cumulative Gain (NDCG)~\cite{Jrvelin2000IREM} computes a score for recommendation $I_k(u)$ which emphasizes higher-ranked true positives.  $D_k(n)=\mathbbm{1}[i_n \in \mathcal{I}_u^+]/\log_2(n+1)$ accounts for a relevancy score. $\mathrm{NDCG}@k=\frac{\mathrm{DCG_k}}{\mathrm{IDCG_k}}=\frac{\sum_{i_n \in I_k(u)}D_k(n)}{\sum_{i_n \in \mathcal{I}_{u,k}^+}D_k(n)}$, where $\mathcal{I}_{u,k}^+$ is the ordered set of top-$k$ true preferred items in descending order of preference. We use the embedding model (MGCCF) presented in~\cite{sun2019a}, which achieves {\em state-of-the-art} performance for this task.

\begin{table}[htbp]
\centering
\caption{Average recall and NDCG for recommender system experiment.}
\setlength{\tabcolsep}{4pt}
\resizebox{\columnwidth}{!}{
\begin{tabular}{lcccc|cccc}
\midrule[0.25ex]
&&\textbf{MGCCF}       &\textbf{EBPR} &\textbf{SGBPR}&&\textbf{MGCCF}  &\textbf{EBPR} &\textbf{SGBPR}  \\
\midrule
\textbf{Recall@10} &{\multirow{4}{*}{\rotatebox[origin=c]{90}{ML100k}}}   &19.91$\pm$0.20 &20.15$\pm$0.07*   &\textbf{20.96$\pm$0.23*} &{\multirow{4}{*}{\rotatebox[origin=c]{90}{Ama.-CDs}}} &13.28$\pm$0.05  &13.21$\pm$0.01* &\textbf{13.83$\pm$0.19*}   \\
\textbf{Recall@20} &&31.27$\pm$0.21 &31.56$\pm$0.07*   &\textbf{32.45$\pm$0.34*} &&20.26$\pm$0.05  &20.37$\pm$0.02*   &\textbf{20.95$\pm$0.22*} \\
\textbf{NDCG@10}  &&26.50$\pm$0.20 &26.86$\pm$0.07*   &\textbf{27.66$\pm$0.24*} &&14.86$\pm$0.04  &14.78$\pm$0.01*   &\textbf{15.31$\pm$0.20*}  \\
\textbf{NDCG@20}   &&31.90$\pm$0.14 &32.23$\pm$0.06*   &\textbf{33.14$\pm$0.31*}  &&18.63$\pm$0.04  &18.65$\pm$0.02   &\textbf{19.20$\pm$0.20*}   \\
\midrule[0.25ex]
\end{tabular}}
\label{table:rec_result}
\end{table}

\noindent \textbf{Discussion}: From Table~\ref{table:rec_result}, we observe that for the recommendation task,
adaptation of the copying model results in a simple intuitive
`EBPR' algorithm which improves recall over the baseline and does not involve any retraining. Use of the sampled graphs to train the model offers significantly better performance.

\section{Conclusion}

We present a novel generative model for graphs called node
copying. The proposed model is based on the idea that neighbourhoods of similar nodes in a graph are similar and can be swapped. It is flexible and can incorporate many existing graph-based learning techniques for identifying node
similarity. Sampling of graphs from this model is simple. The
sampled graphs preserve important structural properties. We have
demonstrated that the use of the model can improve performance for a variety of downstream learning tasks. Future work will investigate developing more general measures of similarity among the nodes of a graph and incorporating them in the copying model, and exploring potential extensions of our theoretical results to more general graphon models.

\appendices
\section{Description of the datasets}\label{ap:data}
\begin{table}[H]
\centering
\caption{Statistics of evaluation datasets for node classification and generative graph models. }
\setlength{\tabcolsep}{4pt}
\resizebox{\columnwidth}{!}{
\begin{tabular}{lccccccc}
\toprule
\textbf{Dataset}      &\textbf{Cora} &\textbf{Citeseer} &\textbf{Pubmed} &\textbf{Coauthor CS} &\textbf{Amazon-Photo} &\textbf{Bitcoin} &\textbf{Polblogs} \\ 
\midrule
\textbf{No. Classes}   &7                              &6                                  &3                                &15                                    &8                                      &2                                 &2                                  \\ 
\textbf{No. Features}  &1,433                          &3,703                              &500                              &6805                                  &745                                    &166                               &N/A                                \\ 
\textbf{No. Nodes}     &2,485                          &2110                               &19,717                           &18,333                                &7,487                                 \ &472                               &1222                               \\ 
\textbf{Edge Density} &0.04\%                         &0.04\%                             &0.01\%                           &0.01\%                                &0.11\%                                 &0.11\%                            &0.56\%                            \\
\bottomrule
\end{tabular}
}
\label{table:dataset_statistics}
\end{table}

\begin{table}[H]
\caption{Statistics of evaluation datasets for the recommender system experiments.}
\centering
\setlength{\tabcolsep}{4pt}
\begin{tabular}{lccccc}
\midrule[0.25ex]
\textbf{Dataset}    &\textbf{Threshold} &\textbf{ \# Users} &\textbf{\# Items} &\textbf{\# Interactions} &\textbf{Density}   \\  \midrule[0.25ex]
\textbf{ML100k}     &10        &897     &823     &52,857         &7.15\%  \\
\textbf{Amazon-CDs} &30        &5,479   &2,605   &137,313        &0.96\%  \\
\midrule[0.25ex]
\end{tabular}
\label{table:rec_datasets}
\end{table} 

We conduct experiments on seven datasets. In the citation datasets (Cora~\cite{sen2008}, Citeseer~\cite{sen2008},  and Pubmed~\cite{namata2012}), nodes represent research papers and the task is to classify them into topics. The graph is built by adding an undirected edge when one paper cites another and features are derived from the keywords of the document. Coauthor CS is a coauthorship graph where each node is an author and two authors are connected if they have coauthored a paper. The node features represent keywords from each author’s papers and the node labels are the active area of research of the authors. Amazon-Photo~\cite{shchur2018} is a portion of the Amazon co-purchase dataset graph from~\cite{mcauley2015}. In this case nodes are products, the features are based on the reviews, and the label is the product type. Items often bought together are linked in the graph. The bitcoin dataset is a transactional network snapshot taken from the dynamic graphs provided in the Elliptic Bitcoin Dataset~\cite{weber2019anti_bitcoin}. In this dataset, nodes represent transactions and edges represent a directed bitcoin flow. Each transaction is labeled as being either fraudulent or legitimate. Polblogs~\cite{adamic2005} dataset is a political blogs network. 
\section{Details of the generative models used for BGCN}\label{ap:bgcndet}
For the node classification experiments comparing different candidate generative models, we use the following hyperparameters and test settings. All of the choices are inherited from the respective original papers. For VGAE~\cite{kipf2016} and GRAPHITE~\cite{grover2019}, we use a two layer GCN with 32 dimensional hidden layer and 16 dimensional encoding. The model is trained for 200 epochs with learning rate 0.01. In DGLFRM-B~\cite{metha2019}, we set $\alpha=50$ and $K=100$.

The encoder network has two nonlinear GCN layers with dimensions 64 and 100. The decoder network has two layers with dimension 32 and 16. The model is trained for 1000 epochs with a learning rate of 0.01 and dropout rate 0.5. All models are trained on 100\% of the dataset (no test edges are held out, because the task of interest is not link prediction). For the MMSBM, we use the same settings that are reported in~\cite{zhang2017}: $\eta = 1, \alpha = 1, \rho = 0.001$. The mini-batch size for SGD is $n =500$ and the hyperparameters related to the decaying step size are  $\epsilon_0 = 1, \tau = 1024$ and $\kappa = 0.5$. For the NETGAN, we set the stopping parameter EO at 50\% using a validation set of 15\% of the links. Due to computation limitations, we used random walks of 10K steps to generate the graphs. The remaining hyperparameters are set to the values reported in~\cite{bojchevski2018}: the generator and the discriminator are respectively set to 40 and 30 layers. The noise is generated by a 16-dimensional multivariate Gaussian distribution. We use the Adam optimizer with a learning rate of 0.0003 with $\lambda_{L_2}=10^{-6}$. The length of the random walks is set to 16. For more details, the complete list of hyperparameters can be found in section H of the supplementary material of~\cite{bojchevski2018}.

\section{BPR for recommendation}\label{appendix:bpr}
In~\cite{rendle2009}, Rendle et al.\ introduce Bayesian Personalized Ranking (BPR) framework
for recommendation systems. Following the notations in Section~\ref{sec:app_rec}, we denote the set of items that are neighbours in the observed
graph for user $u$ as
$\mathcal{I}_{u}^{+}:=\{i \in \mathcal{I}:(u, i) \in \mathcal{G}_{obs}\}$. The training set
can then be written as $D_{S}:=\left\{(u, i, j) | i \in \mathcal{I}_{u}^{+} \wedge j \in \mathcal{I} \setminus \mathcal{I}_{u}^{+}\right\}\,$.\\
In other words, the training set is all triples $(u,i,j)$ such that
user $u$ interacted with $i$ but did not interact with $j$. 
The test set, denoted $\overline{D}_{S}$ comprises all 
triples $(u,i,j)$ such that neither edge $(u,i)$ nor $(u,j)$ appears
in $\mathcal{G}_{obs}$. The goal of the recommender system is to generate a total ranking
$>_u$ of all items for each user $u$.   The relation $i >_u j$ specifies
that user $u$ prefers item $i$ to item $j$. 

In the Bayesian personalized ranking framework
of~\cite{rendle2009}, our task is to
maximize: $
p\left(\mathbf{\Theta} |\{>_{u}\}_{D_{S}}\right) \propto p\left(\{>_{u}\}_{D_{S}} | \mathbf{\Theta}\right) p(\mathbf{\Theta})\,.$
Here $\mathbf{\Theta}$ are the parameters of the model, and $\{>_{u}\}_{D_{S}}$
are the observed preferences in the training data. 
We aim to identify
the parameters $\mathbf{\Theta}$ that maximize this posterior over all users
and all pairs of items. If users are assumed to act independently, then we can write:
\begin{align}
p\left(\{>_{u}\}_{D_{S}} |  \mathbf{\Theta} \right)=\prod_{(u, i, j) \in D_{S}} p\left(i>_{u} j | \mathbf{\Theta}\right)
\end{align}
We define the probability that a user prefers item $i$ over $j$ as
$p\left(i>_{u} j | \mathbf{\Theta}
\right) :=\sigma\left(\hat{x}_{uij}(\mathbf{\Theta})\right).$
Here $\hat{x}_{u i j}$ is a function of the model parameters $\mathbf{\Theta}$ and the observed graph for each triple $(u,i,j)$. In our case, we use the difference between the dot products of the user and item embeddings, so $\hat{x}_{u i j}(\mathbf{\Theta}) = e_u(\mathbf{\Theta})\boldsymbol{\cdot} e_i(\mathbf{\Theta}) - e_u(\mathbf{\Theta}) \boldsymbol{\cdot} e_j(\mathbf{\Theta})$. If we adopt a normal distribution as the prior for $p(\mathbf{\Theta})$ then we can formulate the optimization objective as:
{ 
\begin{align}
\mathrm{BPR}{-}\mathrm{OPT} &:=\ln p\left(\mathbf{\Theta} |\{>_{u}\}_{D_{S}}\right) 
=\sum_{(u, i, j) \in D_{S}} \ln \sigma\left(\hat{x}_{u i j}\right)-\lambda_{\mathbf{\Theta}}||\mathbf{\Theta}||^{2}
\label{eq:BPROPT}
\end{align}}
We maximize this via stochastic gradient descent by repeatedly
drawing triples $(u,i,j)$ randomly from the training set and updating
the model parameters $\mathbf{\Theta}$.

\bibliographystyle{IEEEtran}
\bibliography{references}
\end{document}